\documentclass[preprint,10pt,3p]{elsarticle}

\usepackage[utf8]{inputenc}
\usepackage{amsmath}
\usepackage{amsfonts}
\usepackage{amssymb}
\usepackage{amsthm}
\usepackage{graphicx}
\usepackage{hyperref} 
\usepackage{enumitem}
\usepackage[dvipsnames]{xcolor}
\usepackage{tikz-cd}
\usepackage{geometry}

\usepackage{marginnote}
\usepackage{pgfplots}
\usepackage{subcaption}
\usepackage[titletoc]{appendix}
\usepgfplotslibrary{patchplots,colormaps}
\usepackage[capitalise]{cleveref}
\newtheorem{definition}{Definition}

\newtheorem{assumption}{Assumption}
\newtheorem{lemma}{Lemma}
\newtheorem{proposition}{Proposition}
\newtheorem{corollary}{Corollary}
\newtheorem{remark}{Remark}
\newtheorem{example}{Example}

\Crefname{lemma}{Lemma}{Lemmas}
\crefname{section}{Section}{Sections}
\crefname{assumption}{Assumption}{Assumptions}
\crefname{proposition}{Proposition}{Propositions}
\crefname{corollary}{Corollary}{Corollaries}
\crefname{theorem}{Theorem}{Theorems}
\crefname{definition}{Definition}{Definitions}
\crefname{subsection}{Subsection}{Subsections}

\newcommand{\lp}{\left(}
\newcommand{\rp}{\right)}
\newcommand{\ds}{\displaystyle}
\newcommand{\pd}{\delta}
\newcommand{\pl}{\ell}

\DeclareMathOperator{\SPAN}{Span}
\DeclareMathOperator{\KER}{Ker}

\newcommand{\Span}{\SPAN}
\newcommand{\Ker}{\KER}

\newcommand{\pc}{\mathcal{C}^1_{pc}(0,1)}

\Crefname{appsec}{appendix}{appendices}

\journal{Neural Networks}

\begin{document}

\begin{frontmatter}



\title{A singular Riemannian geometry approach to Deep Neural Networks I. Theoretical foundations.}


\author[ESP,GNCS]{Alessandro Benfenati\corref{cor1}}
\ead{alessandro.benfenati@unimi.it}
\ead{https://sites.unimi.it/a_benfenati/}
\cortext[cor1]{Corresponding Author}

\author[AM,GNFM,INFN]{Alessio Marta}
\ead{alessio.marta@unimi.it}

\affiliation[ESP]{organization={Environmental Science and Policy Department, Università di Milano},
            addressline={Via Celoria 2}, 
            city={Milano},
            postcode={20133}, 
            country={Italy}}
            
\affiliation[GNCS]{organization={Gruppo Nazionale Calcolo Scientifico},
					addressline={INDAM},country={Italy}}   
            
\affiliation[AM]{organization={Dipartimento di Matematica, Universit{\`a} degli Studi di Milano},
            addressline={Via Cesare Saldini 50}, 
            city={Milan},
            postcode={20133}, 
            country={Italy}}
\affiliation[GNFM]{organization={Gruppo Nazionale per la Fisica Matematica},
					addressline={INDAM},country={Italy}}            
\affiliation[INFN]{organization={Istituto Nazionale di Fisica Nucleare, sezione di Milano},
					addressline={INFN},country={Italy}}            

\begin{abstract}
Deep Neural Networks are widely used for solving complex problems in several scientific areas, such as speech recognition, machine translation, image analysis. The strategies employed to investigate their theoretical properties mainly rely on Euclidean geometry, but in the last years new approaches based on Riemannian geometry have been developed. Motivated by some open problems, we study a particular sequence of maps between manifolds, with the last manifold of the sequence equipped with a Riemannian metric. We investigate the structures induced through pullbacks on the other manifolds of the sequence and on some related quotients. In particular, we show that the pullbacks of the final Riemannian metric to any manifolds of the sequence is a degenerate Riemannian metric inducing a structure of pseudometric space. We prove that the Kolmogorov quotient of this pseudometric space yields a smooth manifold, which is the base space of a particular vertical bundle. We investigate the theoretical properties of the maps of such sequence, eventually we focus on the case of maps between manifolds implementing neural networks of practical interest and we present some applications of the geometric framework we introduced in the first part of the paper.
\end{abstract}


\begin{highlights}
\item We consider Neural Networks as sequences of smooth maps between manifolds 
\item We investigate the structures induced through pull-backs on the manifolds 
\item The pull-backs of the final Riemannian metric to any manifolds is a degenerate metric
\item The theoretical study leads to construct equivalence classes in the input manifold
\end{highlights}

\begin{keyword}


Deep Learning\sep Riemann Geometry\sep Classification\sep Degenerate metrics\sep Neural Networks
\end{keyword}

\end{frontmatter}


\section{Introduction}

In the last few years Deep Neural Networks (DNNs) has played a pivotal role in machine learning, both in theoretical studies and practical applications. Thanks to the growing computational power of GPUs in the last decade, the domain of applications of DNNs has extended to a wide variety of complex tasks, as speech recognition \cite{Abdelaziz18,RFP20,LI2021225}, machine translation \cite{Singh17,SYX20,IRANZOSANCHEZ2021303}, image analysis \cite{AN19,LL19}, autonomous driving \cite{HC20} and pattern recognition \cite{Gao20,Bi06}. The vast majority of the theoretical results about Shallow and Deep Neural Networks are obtained in the Euclidean setting, namely viewing input data as points immersed in $\mathbb{R}^n$, for $n$ large enough \cite{HSW89,KT20,Kr20}. This geometric framework implies that there is a natural notion - at least from a geometric point of view - of distance between points which, however, is not always meaningful to describe data because they may live on lower-dimensional manifold, for which the Euclidean distance does not carry useful information. A simple example is that of a horseshoe-shaped surface in $\mathbb{R}^3$: Two points on different extrema of the horseshoe may be close with respect to the Euclidean distance, but are actually far away if one must connect them with a line lying on the surface. Indeed, in many cases of practical interest, one must study data whose underlying structure is that of a non-Euclidean space, the most prominent case is that of data lying on a graph. Some examples are 2D meshed surfaces in computer graphics \cite{AS19,Bos16} or weighted graphs in social networks analysis \cite{Chunaev20}. To be concrete, consider a cloud of points which can be studied using a weighted graph. A notion of distance between points can be built out of the weights, once a measure of similarity or closedness of two points is given - a notable application being PageRank, measuring how two web pages are close. Appealing to the manifold hypothesis, which states that data lie on a low-dimensional manifold embedded in a high dimensional space, it can be convenient to find a low-dimensional representation preserving the most important information before analysing those data. This is the main goal in manifold learning, where one wants to find a low-dimensional representation of a cloud of points starting, for example, from a nearest neighbours graph. Algorithms such as local linear embedding, Laplacian Eigenmap, Diffusion maps, Isomap yield different low-dimensional representations, which may be more or less apt to the task at hand, trying to find the best low-dimensional representation preserving some notion of distance between the points \cite{PJ13,RSE21,TG20}. In particular Diffusion maps and Isomap build an approximation of the (non Euclidean) Riemannian metric of the representation. In supervised training, however, we may not have at our disposal an intrinsic notion of distance between features, but we have to resort to the labels to give a notion of closeness between points. Roughly speaking we can say that two input data are close if so are their labels. For example, in the MNIST handwritten digits dataset, data are embedded in a 784-dimensional Euclidean space, but the Euclidean distance of this high dimensional space carries little information about the contents of the images.  We can say that two images of the digit '5' are way closer that the image of a five and that of a zero. In particular, we can employ the Euclidean distance between labels to say if the corresponding 28$\times$28 images are similar or not. In this framework, if the manifold on which the labels are lying has dimension 10, through the notion of distance between input data previously given, we induced a 10-dimensional metric space on the manifold of initial data, namely a 10-dimensional representation. In literature the set of the techniques generalizing DNNs to non-Euclidean domains have been object of investigation lately, often under the umbrella term of Geometric Deep Learning \cite{BBLSV17,HBL15,MBM17}.

In general, when we are dealing with data whose underlying structure is that of a non-Euclidean space, a convenient geometric setting is that of Riemannian geometry. In this paper we introduce a geometric framework -- employing singular Riemannian geometry -- which can be used to study fully connected neural networks acting on non-Euclidean data. The idea that a neural network can be seen as a sequence of maps between smooth manifolds is at the core of some recent works \cite{Shen18,HaRa17}. The other main idea underlying this work, namely that a neural network can be analysed using Riemannian geometry, was first proposed in \cite{HaRa17}. Our starting point is a definition of neural network as a sequence of maps between manifolds, heavily influenced by \cite{Shen18} and \cite{HaRa17}. In particular, we consider sequences of smooth maps between manifolds of the form
\begin{equation}\label{def:sequence_of_maps_intro}
\begin{tikzcd}
M_{0} \arrow[r, "\Lambda_1"] & M_{1} \arrow[r, "\Lambda_2"] & M_{2} \arrow[r,"\Lambda_3"]  & \cdots \arrow[r,"\Lambda_{n-1}"] & M_{n-1} \arrow[r, "\Lambda_n"] & M_n
\end{tikzcd}
\end{equation}
with the output manifold $M_n$ a Riemannian manifold with metric $g_n$ such that \eqref{def:sequence_of_maps_intro} satisfies suitable hypotheses.
We are interested in studying the structures induced on the manifolds (and on some quotient spaces) by the pullbacks of the Riemannian metric $g_n$ through the maps $\Lambda_i$, which in general are singular (or degenerate) metrics, geometric objects appearing also in other field of differential geometry and of its applications, as in the study of almost-product manifolds or light like submanifolds in general relativity \cite{Bor30,Jan54,Shi74,Bel74,Spe79,BaDu96,Ku96,Sha04}. These degenerate metrics allow us to endow each manifolds of the sequence \eqref{def:sequence_of_maps_intro} with the structure of pseudometric space, from which we can obtain a full-fledged metric space by means of a metric identification. In particular, we show that the quotients $M_{i}/\sim_i$, with $\sim_i$ being the equivalence relation identifying two points of $M_i$ whose pseudodistance is zero, are still smooth manifolds and that the maps $\mathcal{N}_i=\Lambda_n \circ \cdots \circ \Lambda_i$ preserve the lengths of curves, in the sense that, given a curve $\gamma:(0,1)\rightarrow M_i$, then the lengths of $\gamma$ and of $\mathcal{N}_i\circ \gamma$ are the same. The analysis of the sequence in \eqref{def:sequence_of_maps_intro} yields interesting results about deep neural networks on non-Euclidean data. The most notable results of our analysis are that all the points in the same class of equivalence with respect to the equivalence relation $\sim_i$ are mapped by the neural network to the same output and the fact that the pseudodistance between points in the same class of equivalence is zero. 

Beside the theoretical insights beyond these observations and results, there are also interesting applications in real--world problems. 

First, we note that the classes of equivalence of the input manifold $M_0$ carry some information about how the neural network sees the input manifold. In particular if $[x]$ is the equivalence class of a point $x$ in the input manifold $M_0$, then the output of the neural network is the same for every point $y \in [x]$. This shows that a characteristic of neural networks, in addition to perform some complex operations on input data, is to identify points on the input manifold. The knowledge of how points in $M_0$ are identified may be useful, for example, to study how much a classifier is resistant to noise in the labels. 

Another application of the equivalence classes is to build level sets of a neural network trained to perform a non-linear regression task, as proposed in \Cref{subsec:level} (see also  \cite[Section 5.1.1]{BeMa21} for an application in thermodynamics). At last we note that, fixing an input point $x$ in the input manifold $M_0$ and allowing the weights and the biases of the first layer to change, we can see a neural network as the following sequence of maps between manifolds
\begin{equation*}
\begin{tikzcd}
\Omega_{0} \arrow[r, "\Lambda_1"] & M_{1} \arrow[r, "\Lambda_2"] & M_{2} \arrow[r,"\Lambda_3"]  & \cdots \arrow[r,"\Lambda_{n-1}"] & M_{n-1} \arrow[r, "\Lambda_n"] & M_n
\end{tikzcd}
\end{equation*}
where $\Omega_0 = \mathbb{R}^k$, with $k$ a suitable natural number, is the set of values the weights and biases of the first layer can assume. This observation allow us to apply the geometric framework we developed to move along the weights and the biases corresponding with the same minimum, possibly paving the way to find new minima together with usual optimization algorithms, as gradient descent. 


The results of this work naturally lead to develop the Singular Metric Equivalence Class (SiMEC), which builds the class of equivalence of points in the input manifold: the detailed description, the implementation and the experimentation of this algorithm are presented in \cite{BeMa21}. 

The structure of the paper is the following. In \Cref{sec:basic_definitions} we review some basic notions of Riemannian geometry. In \Cref{sec:geometric_framework} we introduce the geometric setting we employ and we give the geometric definitions of smooth layer and smooth neural network. \Cref{sec:main} is devoted to the study of geometric properties of sequences of maps between manifolds in a slightly more general settings compared to the one of \Cref{sec:geometric_framework}. At last, in \Cref{sec:Application to Neural Networks}, we discuss some new interesting results which are direct consequences of the analysis carried out in \Cref{sec:main} and their subsequent application in machine learning problems.
\medskip

\paragraph{Notations.} $\mathbb{R}^n$ is real vector space, whose elements have $n$ elements. $\Ker(f)$ denotes the kernel of the linear application $f$; $A \oplus B$ is the direct sum of $A$ and $B$, $\ds\bigsqcup_iA_i$ denotes the disjoint union of the collection of sets $A_i$; Consider a parametric function $f:\mathbb{R}\rightarrow \mathbb{R}^n$, $t \mapsto \left(f_1(t),f_2(t),\cdots,f_n(t)\right)$, for some $n \in \mathbb{N}$: we say that $f$ is a differentiable function if and only if $\exists \displaystyle\frac{df_i}{dt}(t)\,\forall t \in\text{dom}(f)$, and we denote with $\dot{f}$ the derivative with respect to $t$, namely $\dot{f}(t) = \left(\ds\frac{df_1}{dt}(t),\ds\frac{df_2}{dt}(t),\cdots,\ds\frac{df_n}{dt}(t)\right)$; A map $f$ between two sets $A$ and $B$ is a function from $A$ to $B$. If $f: \mathbb{R}^m \rightarrow \mathbb{R}^m$ is a vector--valued function, we denote the $k$--th component of $f$ with $f_k$; The set of realt matrices is denoted with $\mathbb{R}^{m,n}$. Given a matrix $A$, we denote the element in row $i$ and column $j$ with $A_{ij}$; If $v$ is a vector in $\mathbb{R}^n$, we denote the $i$--th component with $v_i$. If $A$ is a $m \times n$ matrix and $v \in \mathbb{R}^n$ is a vector, then the $i$--th component of $w=Av$ is $w_i = \ds\sum_j A_{ij}v_j$. $A^{(i)}$ denotes the matrix associated to the $i$--th layer of a Neural Network. The notation $\pc$ denotes the set of piece--wise $\mathcal{C}^1$ functions defined on $[0,1]$.

\section{Preliminaries}
\label{sec:basic_definitions}
We begin reviewing some basic notions about manifolds and tangent spaces. We then restate in our notation some classic definitions and results from Riemannian Geometry. The interested reader may find further details in any standard textbook about differential geometry (e.g \cite{doCarmo16}, \cite{Tu11} or \cite{Spi99}).
\subsection{Basic notions in differential geometry}
\label{subsec:basic_differential_geometry}
Intuitively, the notion of smooth manifold can be seen as a generalization of that of surface in $\mathbb{R}^3$. To formalize the notion of manifold, we recall the following definitions.
\begin{definition}
A topological space $\mathcal{T}$ is second countable if there exists a countable collection $\mathcal{U}=\{U_{i}\}_{i=1}^{\infty }$ of open subsets of $\mathcal{T}$ such that any open subset of $\mathcal{T}$ can be written as a union of elements of $\mathcal{U}$.
A topological space $\mathcal{T}$ is a Hausdorff space, or $T_2$, if for any two distinct points $p,q \in \mathcal{T}$ there exist two neighbourhoods $U_p,U_q$ such that $U_p \cap U_q = \emptyset$.
\end{definition}
We define a smooth manifold as follows:
\begin{definition}\label{def:manifold}
A smooth d-dimensional manifold $M$ is a second countable and Hausdorff topological space such that every point $p \in M$ has a neighbourhood $U_p$ that is homeomorphic to a subset of $\mathbb{R}^d$ through a map $\phi_p:U_p \rightarrow \mathbb{R}^{d}$, with the additional requirement that for $p, q\in M$ if $U_p \cap U_q \neq \emptyset$, then $\phi_p \circ \phi_q^{-1}$ is a smooth diffeomorphism. The pair $(U_p,\phi_p)$ is called a local chart and the collection of all the possible local charts at all points is called atlas.
\end{definition}
If we can cover the whole manifold $M$ with a chart, we say that we have a global chart for the manifold, or a global coordinate system.
\begin{example}
The Euclidean space $\mathbb{R}^n$ and the half space $\mathbb{R_+}\times\mathbb{R}^{n-1}$ are trivial examples of manifolds covered by a global chart. The unit circle $S^1$ is an example of a manifold not admitting a global chart.
\end{example}
Given a manifold $M$, a subset $N \subset M$ which is a manifold itself is called submanifold of $M$.
\begin{remark}
The definition given above is not the usual - and more general - one which can be found on the majority of differential geometry books, but it will suffice for our purposes. In particular, in the following we shall always work with manifolds admitting a global chart, which is tantamount to require that all the manifolds will be diffeomorphic to $\mathbb{R}^n$ for some $n \in \mathbb{N}$.
\end{remark}


Given a surface in $\mathbb{R}^3$, we can consider the tangent plane at a point as a linear approximation of the surface; Its counterpart on a manifold is the notion of tangent space at a point. There are several (equivalent) definitions of tangent space at a point. We chose to define it via tangent curves both for its simplicity and because it shall come in handy later.
Let $p \in M$ and consider a chart $(U,\phi)$ containing $p$. A smooth curve $\gamma$ is a curve such that given the chart $(U,\phi)$ the composition $\phi \circ \gamma$ is smooth. Consider two smooth curves $\gamma_1,\gamma_2:[-1,1] \rightarrow U$ such that $\gamma_1(0)=\gamma_2(0)=p$. We say that $\gamma_1$ is equivalent to $\gamma_2$ at $p$ if $\displaystyle{\frac{d}{dt}\left(\phi \circ \gamma_1\right)(t)|_{t=0}=\frac{d}{dt}\left(\phi \circ \gamma_2\right)(t)|_{t=0}}$. See also \Cref{fig:tangent_vector_curves} for a picture of this identification between curves over a point.
\begin{definition}\label{def:tangent_vectors}
A tangent vector $\dot{\gamma}(0)$ over $p$ is an equivalence classes of curves; The tangent space of $M$ at $p$, which is denoted with $T_p M$, is the set of all tangent vectors at $p$. A choice of a basis for $T_pM$ is called reference frame.
\end{definition}
\begin{figure}[htbp]
\begin{center}
\includegraphics[width=0.5\textwidth]{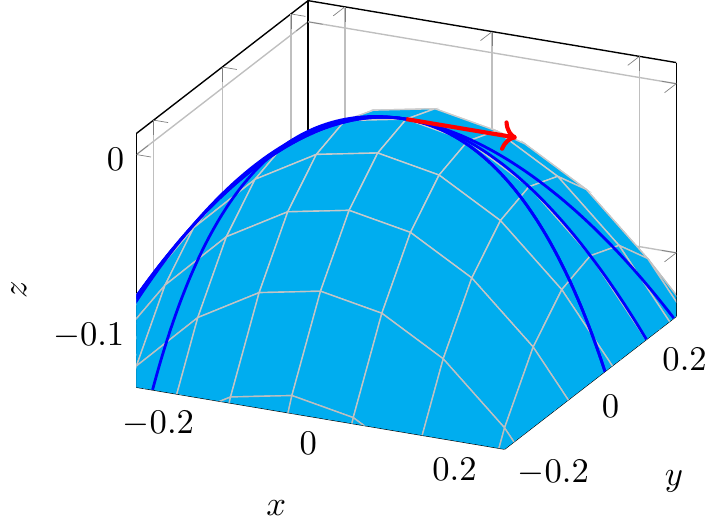}
\end{center}
\caption{Three different curves individuate the same tangent vector $v$ over a point $p$ of a paraboloid. To define the tangent vector $v$ we can use any of the three curves; In this sense the curves are equivalent at $p$.}
\label{fig:tangent_vector_curves}
\end{figure}

Note that the definition of tangent vector is not dependent on a particular choice of a chart. 
Taking the disjoint union of tangent spaces at all points, we obtain the tangent bundle $TM$, which is a manifold itself. 

\begin{example} Consider two classical cases: $M=\mathbb{R}^n$ and $M=S^1$:\begin{itemize}
\item $M=\mathbb{R}^n$: consider $M$ as an affine space. $M$ is a smooth manifold whose tangent space $T_p M$ over each point coincides with the vector space $\mathbb{R}^n$, namely the space of the displacement vectors. The tangent bundle is given by $TM = M \times \mathbb{R}^n$.
\item $M=S^1$: The tangent bundle is given by the disjoint union of all the tangent lines.
\end{itemize}
\end{example}
Since $TM$ is a manifold itself, we can introduce the notion of smooth vector field as follows.

\begin{definition}
A smooth vector field is a smooth curve $v:(a,b) \subseteq \mathbb{R}\rightarrow TM $ which associates $t \in (a,b)$ to a vector $v_p$ in $T_p M$ in a smooth way. 
\end{definition}
Using vector fields, one can build other geometric objects, like distributions and vertical bundles. These notions are not actually fundamental for the fully understanding of the idea beyond this work, but they play a role in \Cref{prop:vertical_integrable} and \Cref{prop:equivalence_characterization}, hence we included some technical details about these mathematical objects in \ref{appendix} for sake of completeness.


To conclude this section, we recall the notion of a differential of a map between two manifolds $M$ and $N$, which is a linear application transporting tangent vectors of $M$ to tangent vectors of $N$. 

\begin{definition} Let $M$ and $N$ be two smooth manifolds and let $\phi: M \rightarrow N$ be a smooth map. Consider a curve $\alpha:(-\varepsilon,\varepsilon) \rightarrow M$ with $\alpha(0) = p \in M$ and $\alpha^\prime(0)=w$. Then we can define the curve $\beta=\phi\circ\alpha$, with $\beta(0)=\phi(p)\in N$ and such that $\beta^\prime(0) \in T_{\phi(p)}N$. The map $d \phi_p : T_p M \rightarrow T_{\phi(p)}N$ defined by $d \phi_p(w) = \beta^\prime(0)$ is a linear map called differential of $\phi$ at $p \in M$.
\end{definition}

\subsection{Riemannian geometry}
\label{subsec:riemannian_geometry}
We begin reviewing some basic notions of Riemannian geometry, focusing on the case of manifolds realized as subsets of $\mathbb{R}^n$. The pivotal notion in Riemannian geometry is that of Riemannian metric, defined on a generic manifold as follows.
\begin{definition}\label{def:general_metric}
A Riemannian metric $g$ over a smooth manifold $M$ is a smooth family of inner products on the tangent spaces of $M$; Namely, $g$ associates to every $p \in M$ an inner product $g_p: T_p M \times T_p M \rightarrow \mathbb{R}$. 
\end{definition}
Given a metric $g$, we define the norm of a vector $v \in T_p M$ as $\| v \|_p = \sqrt{g_p(v,v)}$. When a manifold satisfies \Cref{as_1} (see \Cref{sec:geometric_framework}) we can identify $T_p \mathbb{R}^n$ with $\mathbb{R}^n$ itself \cite{Tu11}. In particular we can simplify \Cref{def:general_metric} to the following one.
\begin{definition}\label{def:simplified_metric}
A Riemannian metric $g$ over a smooth manifold $M$ as in \cref{as_1} is a map $g:M \rightarrow Bil(\mathbb{R}^n \times \mathbb{R}^n)$ that associates to each point $p$ a positive symmetric bilinear form $g_p:\mathbb{R}^n \times \mathbb{R}^n \rightarrow \mathbb{R}$ in a smooth way.
\end{definition}
Note that $g_p$ is positive definite: $g_p(x,y)=0$ if and only if $x=0$ or $y=0$.
\begin{remark}
Even if we can specialize definition \Cref{def:general_metric} to \Cref{def:simplified_metric}, it is important not to confuse $M$ and its tangent space. Since one should think of $M=\mathbb{R}^n$ as an affine space and of $T_p M \simeq \mathbb{R}^n$ as the space of displacement vectors, a metric $g$ associates to every point $x \in M$ a bilinear form over displacement vectors.
\end{remark}
Given a coordinate system $X$, providing the matrix associated to $g$ in such coordinate system fully specifies the metric. Let $Y$ be another coordinate system related to the original one by the diffeomorphism $\varphi:M \rightarrow M$, $(y_1,\cdots,y_n) \mapsto (x_1,\cdots,x_n)$. Then the matrix representing the metric $g$ in the new coordinates is given by
\begin{equation}\label{eq:metric_change_of_charts} 
g^Y_{ij} =  \sum_{h,k}\left( \frac{\partial x^h}{\partial y^i} \right) g^X_{hk} \left( \frac{\partial x^k}{\partial y^j} \right)
\end{equation}
where $\ds\frac{\partial x^h}{\partial y^i} = (J_\varphi)_{hi}$ is the $(h,j)$--entry of the Jacobian of the function $\varphi$.
\begin{example}\label{ex:canonical_metric_r2}
The canonical Euclidean metric of $\mathbb{R}^2$ in Cartesian coordinates $(x,y)$ is represented by  the matrix
\begin{equation}
g^C = 
\begin{pmatrix}
1 & 0 \\
0 & 1
\end{pmatrix}
\end{equation}
Let $M=(0,\infty)\times (0,\infty)$ be the first quadrant of the plane. In polar coordinates $(r,\theta) \in (\mathbb{R}_0^+ \times(0,\pi/2))$, the matrix $g^P$ associated to $g$ reads
\begin{equation}
g^P = 
\begin{pmatrix}
1 & 0 \\
0 & r^2
\end{pmatrix}
\end{equation}
Note that the matrix $g^P$ is a smooth function of the variable $r$.
\end{example}
\Cref{ex:canonical_metric_r2} also shows a general fact: If a manifold $M$ can be realized as a subset of the Euclidean space, it naturally inherits a canonical Riemannian metric form $\mathbb{R}^n$. Let $	\iota : M \rightarrow \mathbb{R}^n$ be the immersion map of $M$ in $\mathbb{R}^n$. Then we can equip $M$ with the metric
\begin{equation}\label{eq:induced_metric_rn} 
h = J_\iota^T g J_\iota
\end{equation}
where $J_\iota$ is the Jacobian matrix of $\iota$ in local coordinates, namely $(J_\iota)_{ij} = \left( \dfrac{\partial \iota^i}{\partial x^j} \right)$.

\begin{example}
The canonical metric of $S^2$ induced by that of $\mathbb{R}^3$ is represented, in spherical coordinates $(\phi,\theta) \in (0,2\pi) \times (0,\pi)$, by the matrix
\begin{equation}\label{eq:induced_metric_S2} 
h = 
\begin{pmatrix}
1 & 0 \\
0 & \sin^2(\theta)
\end{pmatrix}
\end{equation}
In fact, we can realize $S^2$ as a subset of $\mathbb{R}^3$ through the immersion 
\begin{eqnarray*}
&\iota :& (0,2\pi) \rightarrow \mathbb{R}^3\\
&\iota :& (\phi,\theta) \mapsto 
\begin{pmatrix}
\sin (\theta)\cos (\phi) \\
\sin (\theta)\sin (\phi) \\
\cos (\theta) 
\end{pmatrix}
\end{eqnarray*}
Applying \Cref{eq:induced_metric_rn}, with $g$ the canonical metric of $\mathbb{R}^3$, yields \Cref{eq:induced_metric_S2}.
\end{example}

\begin{remark}
Two subset of $\mathbb{R}^n$ endowed with different metrics have different geometries, in the sense that they are describing two different geometric shapes. For example consider the set $S = (0,1) \times (0,1)$ with the canonical Euclidean metric $g^C$ as in \Cref{ex:canonical_metric_r2}, then $S$ represents a flat surface, a square of the real plane. If we consider the same set $S$ with the metric $g^P$ defined in \Cref{eq:induced_metric_S2}, then $S$ represents part of the surface of a sphere.
\end{remark}

If a manifold $M$ is equipped with a Riemannian metric $g$, we have a canonical definition of length of a curve. 
\begin{definition}\label{def:lenght}
Let $\gamma:[a,b] \rightarrow M$ be a piecewise smooth curve, then its length is 
\begin{equation}
L(\gamma) = \int_a^b \| \dot{\gamma}(s) \|_{\gamma(s)} ds = \int_a^b \sqrt{g_{\gamma(s)}(\dot{\gamma}(s),\dot{\gamma}(s))} ds
\end{equation}
\end{definition}

\begin{remark}
Notice that, given a curve $\gamma:[a,b] \rightarrow M$, $L(\gamma)=0$ if and only if $\gamma$ is the constant map associating to every $s \in [a,b]$ the same point $p\in M$. This is a consequence of the non-degeneracy of the Riemannian metric.
\end{remark}
If a manifold $M$ is path--connected, as in \Cref{as_1}, we can also define the energy of a curve and a distance function $d(x,y)$.

\begin{definition}
Let $\gamma$ a curve on a path--connected manifold $M$. The energy of $\gamma$ is
\begin{equation}
E(\gamma) = \int_0^1  \| \dot{\gamma}(s) \|^2_{\gamma(s)} ds =  \int_a^b g_{\gamma(s)}(\dot{\gamma}(s),\dot{\gamma}(s)) ds
\end{equation}
\end{definition} 
 In Riemannian geometry $L(\gamma)= 0$ if and only if $E(\gamma)=0$ and, in general, a curve minimizes the length functional if and only if it minimizes the energy functional.

\begin{definition}\label{def:distance}
Consider  two distinct points $x$ and $y$ on a path-connected manifold $M$. The distance $d(x,y)$ between such points is a function $d:M \times M \rightarrow \mathbb{R}_0^+$ defined as
\begin{equation}
\begin{split}
d(x,y) = \inf \{ L(\gamma) \ | \ \gamma:[0,1]\rightarrow M, \gamma\in\pc \mbox{ s.t. } \gamma(0)=x \mbox{ and } \gamma(1)=y \}   
\end{split}
\end{equation}
The pair $(M,d)$ is a metric space.
\end{definition}


Consider two smooth manifolds $M,N$ and let $\Lambda:M \rightarrow N$ be a smooth map. Suppose that $N$ is equipped with a Riemannian metric $g^N$. Then we can endow $M$ with the pullback metric $g^M = \Lambda^* g^N$, where $\Lambda^*$ is the pullback through $\Lambda$. Chosen two global coordinate systems $(x_1,\cdots,x_m)$ and $(y_1,\cdots,y_n)$ of $M$ and $N$ respectively, the matrix associated to the pullback of $g^N$ through $\Lambda$ reads:
\begin{equation}\label{eq:metric_pullback}
\left(g^M\right)_{ij} = \sum_{h,k = 1}^{\dim(N)} \left( \frac{\partial \Lambda^h}{\partial x^i} \right) (g^N)_{hk} \left( \frac{\partial \Lambda^k}{\partial x^j} \right)
\end{equation}

\begin{remark}
Notice that \Cref{eq:metric_pullback} encompasses \Cref{eq:metric_change_of_charts} and \Cref{eq:induced_metric_rn} as particular cases.
\Cref{eq:metric_pullback} also yields a way to explicitly compute the Riemannian metric of manifolds realized as subsets of other Riemannian manifolds.
\end{remark}
When $\Lambda:M \rightarrow N$ is a diffeomorphism or $\Lambda$ is an immersion, encompassing the case in which $\dim(M)<\dim(N)$ and $J_\Lambda$ is injective, then the metric $g^M = \Lambda^* g^N$ obtained using the pullback of $g^N$ through $\Lambda$ is still non-degenerate; However, if $J_\Lambda$ is not injective, and this is certainly true if $\dim(M)>\dim(N)$, then $g^M = \Lambda^* g^N$ is degenerate - the matrix associated with $g^M$ is not of full rank - and therefore it is not a Riemannian metric.

Consider \eqref{def:sequence_of_maps_intro} with $n$=2:
\begin{equation*}
\begin{tikzcd}
M_0 \arrow[r, "\Lambda_1"] & M_{1} \arrow[r, "\Lambda_2"] & M_{2} 
\end{tikzcd}
\end{equation*}
and denote with $g^{(i)}$ the Riemann metric on manifold $M_i$: then, using \eqref{eq:metric_pullback}, one is able to compute the pullback of $g^{(2)}$ on $M_0$:
\begin{eqnarray*}
(g^{(0)})_{ij} &=& \sum_{h,k = 1}^{\dim(M_2)} \left( \frac{\partial \Lambda^h}{\partial x^i} \right) (g^{(1)})_{hk} \left( \frac{\partial \Lambda^k}{\partial x^j} \right) \\
&=& \sum_{h,k = 1}^{\dim(M_2)} \left( \frac{\partial \Lambda^h}{\partial x^i} \right) \left[\sum_{p,q = 1}^{\dim(M_1)} \left( \frac{\partial \Lambda^p}{\partial x^h} \right) (g^{(2)})_{pq} \left( \frac{\partial \Lambda^q}{\partial x^k} \right) \right]\left( \frac{\partial \Lambda^k}{\partial x^j} \right).
\end{eqnarray*}
This can be thus generalized on sequences such as \eqref{def:sequence_of_maps_intro}. 

\subsection{Singular Riemannian geometry}\label{subsec:singular_riemannian_geometry}
Finally, we introduce the notion of singular Riemann metric which will be useful to analyse smooth fully connected neural networks using the geometric framework introduced in \eqref{subsec:geometric_framework_neural}.
\begin{definition}\label{def:general_singular_metric}
A singular Riemannian metric $g$ over a smooth manifold $M$ is a smooth family of positive semidefinite symmetric bilinear forms on the tangent spaces of $M$. 
\end{definition}
Given a singular metric $g$, we define the seminorm of a vector $v \in T_p M$ as $\| v \|_p = \sqrt{g_p(v,v)}$. As we did for the definition of Riemannian manifold, we can specialize the definition of singular Riemannian metric to manifolds abiding to \Cref{as_1}.
\begin{definition}\label{def:simplified_singular_metric}
A singular Riemannian metric $g$ over a smooth manifold $M$ as in \Cref{as_1} is a map $g:M \rightarrow Bil(\mathbb{R}^n \times \mathbb{R}^n)$ that associates to each point $x$ a positive semidefinite symmetric bilinear form $g_x:\mathbb{R}^n \times \mathbb{R}^n \rightarrow \mathbb{R}$ in a smooth way.
\end{definition}
All the formulae of the previous section continue to hold true even in the case of a singular Riemannian metric, with the difference that we may have smooth non-constant curves - namely those curves whose image is a single point - of null length. The degeneracy of the metric induces a decomposition of the tangent space $T_p M$ over every point $p \in M$ as a direct sum $T_p P \oplus T_p N$, with $T_p P$ the subspace of vectors with positive seminorm and $T_p N$ the subspace spanned by the vectors whose seminorm is zero. If we define the pseudolenght $\ell$ and the pseudodistance $\delta$ as in \Cref{def:lenght} and \Cref{def:distance}, the previous observation about the existence of non-trivial curves of length zero in $M$ entails that there are points whose distance is null or, in other words, the pair $(M,\delta)$ is a pseudometric space. Identifying the metrically indistinguishable points using the equivalence relation $x \sim y \Leftrightarrow \delta(x,y)=0$ for $x,y \in M$, we obtain the metric space $(M / \sim,\delta)$. Note that a class of equivalence $[x]$ is the set $\{y \in M \ | \ \delta(x,y) = 0\}$ and therefore the points of $M/\sim$ are classes of equivalence of points in $M$.

\begin{example}\label{ex:singular_pullback_linear}
Let $A:\mathbb{R}^3\rightarrow \mathbb{R}^2$ be the linear function which in the canonical basis is represented by the matrix
\begin{equation}
A= \begin{pmatrix}
1 & 2 & 2 \\
3 & 1 & 5
\end{pmatrix}
\end{equation}
Suppose to endow $\mathbb{R}^2$ with its canonical metric $g$ given in example \eqref{ex:canonical_metric_r2}. Then we can equip $\mathbb{R}^3$ with a singular Riemannian metric $h$ using the pullback of $g$ through $A$. Working in Cartesian coordinates for both $\mathbb{R}^2$  and $\mathbb{R}^3$, the matrix associated to $h$ is:
\begin{equation}
\begin{pmatrix}
10 & 5 & 17 \\
5 & 5 & 9 \\
17 & 9 & 29
\end{pmatrix}
\end{equation}
which is of rank $2$. A straightforward computation yields that $\Ker(h)=\Span\{(8,1,-5) \}$. Applying \cref{def:distance} we find that $x \sim y$ if and only if $x-y \in \Ker(h)$. In particular the equivalence classes are straight lines parallel to $\Span\{(8,1,-5) \}$. Indeed, let $x-y \in \Ker(h)$ and consider the oriented segment from $x$ to $y$ parametrized by $\gamma(s)=x+s(y-x)$ with $s \in [0,1]$. Then $\dot{\gamma}(s)=y-x \in \Ker(h_{\gamma(s)})$ for every $s \in [0,1]$. Therefore we have that $l(\gamma)=\int_0^1 \sqrt{h_{\gamma(s)}(y-x,y-x)} ds = 0$ and we proved that the pseudodistance between $x$ and $y$ is zero, namely that they are in the same class of equivalence. On the other hand if $x \sim y$ then $\delta(x,y)=0$, which means that there is a $\mathcal{C}^1$ curve $\gamma:[0,1]\rightarrow \mathbb{R}^3$ with $\gamma(0)=x$ and $\gamma(1)=y$ such that $\dot{\gamma}(s) \in \Ker(h)$ for every $s \in [0,1]$. In particular we can consider $\dot{\gamma}(s)=v \in \Ker(h)$, which along with the conditions at $s=0$ and $s=1$ gives the curve $\gamma(s)=x+s(y-x)$, a segment lying on a line parallel to $\Span\{(8,1,-5) \}$.
\end{example}


\section{A geometric framework for neural network}
\label{sec:geometric_framework}

\subsection{General assumptions}
As we anticipated in the introduction, we propose a geometric definition of neural network as a sequence of smooth maps $\Lambda_i$ between manifolds $M_j$ of the form:
\begin{equation}\label{def:sequence_of_maps}
\begin{tikzcd}
M_{0} \arrow[r, "\Lambda_1"] & M_{1} \arrow[r, "\Lambda_2"] & M_{2} \arrow[r,"\Lambda_3"]  & \cdots \arrow[r,"\Lambda_{n-1}"] & M_{n-1} \arrow[r, "\Lambda_n"] & M_n
\end{tikzcd}
\end{equation}
We call $M_0$ the \textit{input manifold} and $M_n$ the \textit{output manifold}. All the other manifolds of the sequence are called \textit{representation manifolds}. In this section we state some general assumptions on this sequence and we introduce the main geometric objects we need in the following. The first hypothesis is the following.
\begin{assumption}\label{as_1}
The manifolds $M_i$ are open and path-connected sets of dimension $\dim M_i=d_i$.
\end{assumption}
The next two definitions are propaedeutic to the second assumption we make.
\begin{definition}\label{def:submersion}
Let $f:M \rightarrow N$ be a smooth map between manifolds. Then $f$ is a submersion if, in any chart, the Jacobian $J_f$ has rank $\dim(N)$.
\end{definition}
\begin{remark}
In literature, see e.g \cite{DFO20}, there is also another stricter definition requiring that the differential $d_f$ must be surjective at every point. In our case, under \Cref{as_2} below, the two definitions coincide.
\end{remark}
\begin{definition}\label{def:embedding}
Let $f:M \rightarrow N$ be a smooth map between manifolds. $f$ is an embedding if its differential is everywhere injective and if it is an homeomorphism with its image. In other words, $f$ is a diffeomorphism with its image.
\end{definition}
\begin{remark}
An embedding allow to realize a manifold as a subset of another space. 
\end{remark}
\begin{assumption}\label{as_2}
The sequence of maps \eqref{def:sequence_of_maps} satisfies the following properties:
\begin{itemize}
\item[1)] If $\dim(M_{i-1}) \leq \dim(M_i)$ the map $\Lambda_i: M_{i-1} \rightarrow M_i$ is a smooth embedding.
\item[2)] If $\dim(M_{i-1}) > \dim(M_i)$ the map $\Lambda_i: M_{i-1} \rightarrow M_i$ is a smooth submersion.
\end{itemize}
\end{assumption}
Our last general hypothesis is the following:
\begin{assumption}\label{Assumption_2}
The manifold $M_n$ is equipped with the structure of Riemannian manifold, with metric $g^{(n)}$.
\end{assumption}
The pullbacks of $g^{(n)}$ through the maps $\Lambda_n$, $\Lambda_n \circ \Lambda_{n-1}$, ..., $\Lambda_n \circ \Lambda_{n-1} \circ \cdots \circ \Lambda_{1}$ yield a sequence of (generally degenerate) Riemannian metrics $g^{(n-1)}, g^{(n-2)},\cdots, g^{(0)}$ on $M_{n-1},M_{n-2},\cdots,M_0$. Hereafter, we shall denote with $\mathcal{N}_i$ the map $\Lambda_n \circ \Lambda_{n-1} \circ \cdots \circ \Lambda_{i}$. In the case $i=1$, we shall simply write $\mathcal{N}$ instead of $\mathcal{N}_1$.

The singular Riemannian metrics $g^{(i)}$ allow us to decompose the tangent space over each point of $M_i$ as the direct sum of two vector spaces. Let $x \in M_{i}$, then we can write $T_x M_{i} = T_x P_{i} \oplus T_x N_{i}$, with $T_x P_{i}$ the subspace of all the vectors $v \in T_x M_{i}$ such that $g^{(i)}_x(v,v)>0$ and $T_x N_{i}$ the subspace of all the vectors $v \in T_x M_{i}$ such that $g^{(i)}_x(v,v)=0$. We shall call the vectors $v \in T_x M_{i}$ \textit{null vectors or vertical vectors} if $g^{(i)}_x (v,v) = 0$, while an element $v \in T_x P_{i}$ will be called \textit{spacelike} or \textit{horizontal vector}. The terms null and spacelike vectors are borrowed from general relativity.

\subsection{A geometric framework for neural networks}\label{subsec:geometric_framework_neural}
Now we present the geometric framework we will use to study fully connected neural networks. We begin by giving some definitions and stating some assumptions influenced by those given in \cite{HaRa17} and \cite{Shen18}.

\begin{assumption}\label{Assumption_N1}
We assume that the manifolds $M_i$ are diffeomorphic to $\mathbb{R}^{d_i}$.
\end{assumption}

In this framework, we are not using the inner product of induced by $\mathbb{R}^{d_i}$. Note that in practical applications the manifolds $M_i$ are subsets of $\mathbb{R}^{d_i}$ of the form $(a_1,b_1)\times \cdots (a_{d_i},b_{d_i})$ with $-\infty \leq a_i < b_i \leq + \infty$.

Motivated by \cite{Shen18}, we define the following notion of smooth feed--forward layer, a particular kind of maps between manifolds.
\begin{definition}[Smooth layer]\label{def_layer}
Let $M_{i-1}$ and $M_i$ be two smooth manifolds satisfying \cref{Assumption_N1}. A map $\Lambda_i : M_{i-1} \rightarrow M_i$ is called a smooth layer if it is the restriction to $M_{i-1}$ of a function $\overline{\Lambda}^{(i)}(x)  : \mathbb{R}^{d_{i-1}} \rightarrow \mathbb{R}^{d_i}$ of the form
\begin{equation}
\overline{\Lambda}^{(i)}_\alpha (x)= F^{(i)}_\alpha\lp\sum_\beta A_{\alpha\beta}^{(i)} x_\beta+b^{(i)}_\alpha\rp 
\end{equation}
for $i=1,\cdots,n$, $\alpha=1,\cdots, d_{i-1}$, $x \in \mathbb{R}^{d_{i}}$, $b^{(i)} \in \mathbb{R}^{d_i}$ and $A^{(i)} \in \mathbb{R}^{d_{i} \times d_{i-1}}$, with $F^{(i)} : \mathbb{R}^{d_i} \rightarrow \mathbb{R}^{d_i} $ a diffeomorphism.
\end{definition}

The matrices $A^{(i)}$ and the vectors $b^{(i)}$ associated with the $i$--th layer are called matrices of weights and biases of the $i$--th layer, respectively. The functions $F^{(i)}$ are called activation functions \cite{GBC16,Agg18}. We remark that the layers maps $\Lambda_i$ are open maps and that they satisfy \Cref{as_2}. 
\begin{example}[Feedforward sigmoid layer]\label{ex:fsl}
A notable example of layer, usually referred as feedforward layer with sigmoid activation function in the neural network literature, is the following \cite{Agg18,DFO20}. Consider the sets $M_1 = (a,b)^{d_1}$, with $a,b \in \mathbb{R}$, and $M_2 = (0,1)^{d_2}$. Consider the sigmoid function 
\begin{eqnarray*}
&\sigma:& \mathbb{R} \rightarrow (0,1)\\
&\sigma:& x \mapsto \dfrac{e^x}{1+e^x}
\end{eqnarray*}
Let $A^{(1)} \in \mathbb{R}^{d_1 \times d_2}$. The feed--forward layer with sigmoid activation function is the map $\Lambda: M_1 \rightarrow M_2$ as that function that, given a point $p=(p_{1},\cdots,p_n) \in M_1$, maps $p$ to 
$$
\Lambda(p) = \lp \sigma\Big(\sum_j A_{1j}^{(1)} p_j\Big),\sigma\Big(\sum_j A_{2j}^{(1)} p_j\Big),\cdots,\sigma\Big(\sum_j A_{d_2j}^{(1)} p_j\Big) \rp
$$
\end{example}
We also assume the following hypothesis:
\begin{assumption}[Full rank hypothesis]\label{Assumption_FR}
We assume that all the matrices of weights $A^{(i)}$, $i=1,\cdots,n$, are of full rank.
\end{assumption}
\begin{remark}
Consider the set of all the $m\times n$ matrices. If the entries of these matrices are generated using a probability distribution $F$ whose associated measure $\mu_F$ is absolutely continuous with respect to the Lebesgue measure - for example the uniform and the Gaussian distributions - then the matrices which are not of full rank are a subset of $\mathbb{R}^{m,n}$, the set of $m\times n$ real matrices, of null measure. In fact, let $A \in \mathbb{R}^{m,n}$ and set $k= \min \{m,n\}$. Let $\{S_i\}$ be the set of all the $k \times k$ submatrices of $A$. Define the function $f(A)=\sum_i \det(S_i)^2$, which is a polynomial function of the entries of $A$. Since in general $f$ is not the identically zero polynomial, then the set of the zeroes of $f$ has zero Lebesgue measure and $A$ is full-rank except on a set of Lebesgue measure $0$.
\end{remark}

\Cref{as_1,as_2} are satisfied if we assume \Cref{Assumption_N1} and that the maps $\Lambda_i$ are as in \Cref{def_layer}, once the full rank hypothesis \Cref{Assumption_FR} is satisfied. Indeed, \Cref{Assumption_N1} immediately entails \Cref{as_1}, while the full rank hypothesis and the properties of the maps $\Lambda_i$ in \Cref{def_layer} yield that the functions $\Lambda_i$ are smooth submersions if $\dim(M_{i-1}) > \dim(M_i)$. The second condition in \Cref{as_2} is satisfied in the case $\dim(M_{i-1}) \leq \dim(M_i)$ for $\Lambda_i$ is smooth and because it is a diffeomorphism with its image, being the matrix of weights of $\Lambda_i$ of full rank, with $F_\alpha$ a diffeomorphism.
\begin{remark}
The hypotheses of \cref{Assumption_N1} are satisfied by neural network maps whose layers are defined using some commonly employed activation functions \cite{Agg18,GBC16,DFO20}, for example:
\begin{enumerate}[label=\arabic*)]
\item In a sigmoid layer $F :\mathbb{R}^{l} \to \mathbb{R}^{l}$, 
$F^j(x)=\dfrac{1}{1+e^{-x_j}} {\textit{ for }}j=1,\cdots,l$.
\item[]
\item In a softplus layer $F :\mathbb{R}^{l} \to \mathbb{R}^{l}$
$F^j(x)=\ln(1+e^{-x_j}) {\textit{ for }}j=1,\cdots,l$.
\item[]
\item In a softmax layer, $F :\mathbb{R}^{l} \to \mathbb{R}^{l}$
$F^j(x)=\dfrac {e^{x_j}}{\ds\sum_{k=1}^{l}e^{x_k}} {\textit{ for }}j=1,\cdots,l$.

\end{enumerate}
An example of commonly employed activation function not encompassed by our definition of layer is ReLu. In a ReLu layer $f :\mathbb{R}^{l} \to \mathbb{R}^{l}$,  $F(x)_j= \max\{0,x_j\}$ is not even $\mathcal{C}^1$. However we note that, softplus layers and SiLu layers are used as smooth approximations of ReLu layers.
\end{remark}

The following definition extends that of Neural Networks in our framework, and the subsequent examples show  how classical network structures, arising in practical applications, satisfy indeed our definition.
\begin{definition}[Smooth Neural network]\label{def:smooth_neural_network}
A smooth neural network is a sequence of maps between manifolds 
\begin{equation}
\begin{tikzcd}
M_{0} \arrow[r, "\Lambda_1"] & M_{1} \arrow[r, "\Lambda_2"] & M_{2} \arrow[r,"\Lambda_3"]  & \cdots \arrow[r,"\Lambda_{n-1}"] & M_{n-1} \arrow[r, "\Lambda_n"] & M_n
\end{tikzcd}
\end{equation}
with $n \geq 2$ such that the manifolds $M_i$ abide to hypothesis \eqref{Assumption_N1} and the maps $\Lambda_i$ are as in definition \eqref{def_layer}.
\end{definition}

\begin{example}
Classical structures satisfy the above definition:
\begin{itemize}
\item {\normalfont Shallow Network:} The sequence of maps 
$$
\begin{tikzcd}
M_{0} \arrow[r, "\Lambda_1"] & M_{1} \arrow[r, "\Lambda_2"] & M_{2}
\end{tikzcd}
$$
with $\Lambda_1$ and $\Lambda_2$ as in \cref{def_layer} is a shallow neural network.
\item {\normalfont Deep Network:} A sequence of maps \eqref{def:sequence_of_maps} with more than two layers, namely with $n \geq 3$, whose maps $\Lambda_i$ abiding to \cref{def_layer} is a deep neural network.
\end{itemize}
\end{example}

Up to now we never mentioned the notion of node. In our framework $d_i = \dim(M_i)$ is the intrinsic dimension of a representation space, not the dimension of the space in which we realize $M_i$. For example, we can describe the position on the unit sphere $S^2$ using three coordinates, but this does not mean that the sphere is a manifold of dimension $3$. In fact, the unit sphere $S^2$ is actually a two dimensional manifold which can be embedded in $\mathbb{R}^3$. In this embedding the three coordinates $x,y,z$ are related by $x^2+y^2+z^2=1$, so only two coordinates are independent. In our framework a node is a coordinate of the higher dimensional embedding space $\mathbb{R}^k$, with $k \geq d_i$, in which we embed the representation manifold $M_i$. When $k=d_i$, then the number of nodes correspond to the dimension of the manifold, but in general $k\geq d_i$ and the number of nodes is greater than the dimension of $M_i$. Indeed, due to the Whitney embedding theorem we know that given any $m$-dimensional manifold $M$, $m \in \mathbb{N}$, then we can realize $M$ in the Euclidean space $\mathbb{R}^{2m}$. The dimension of the embedding space may be lower -- as in the case of the sphere $S^2$ which can be realized in $\mathbb{R}^3$ -- but it is an analysis to be made case by case. We illustrate the difference between the dimension of a manifold of the sequence \eqref{def:sequence_of_maps_intro} and the number of nodes, namely the dimension of the embedding space, with the following example.
\begin{example}
Consider a sigmoid layer $\Lambda_1: M_0 \rightarrow M_1$ without biases, $M_0=\mathbb{R}^2$ and $M_1=\Lambda_1(M_0)$. Let $\sigma : \mathbb{R} \rightarrow (0,1)$ be the sigmoid activation function and consider the matrix of weights
\begin{equation*}
A_1 =
\begin{pmatrix}
1 & 2\\
3 & 4\\
5 & 6
\end{pmatrix}
\end{equation*}
If $(x_0,x_1)$ are global coordinates on $\mathbb{R}^2$, then
$$\Lambda_1(x_0,x_1)=\left(\sigma(x_0+2x_1),\sigma(3x_0+4x_1),\sigma(5x_0+6x_1) \right)$$
The fact that the output of $\Lambda_1$ is a three-dimensional vector does not mean that $M_1=\Lambda_1(M_0)$ is a manifold of dimension $3$. Remember that, as \Cref{def_layer} and \Cref{ex:fsl}, a sigmoid layer is a composition of two map, with the linear application represented by the matrix $A_1$ acting directly on $M_0$ and, whose output is then fed to the activation functions componentwise, through the map $F_1$ as per \cref{def_layer}. The range of the linear application $A_1$ is a two-dimensional vector space and by hypothesis $F_1$ is a diffeomorphism. This means that $F_1(A_1(M_0))=\Lambda_1(M_0)$ is a two-dimensional space and in particular that we can use $(x_0,x_1)$ as a global chart for $M_1$. Therefore, the output of the layer $\Lambda_1$ is a 2-dimensional manifold embedded in a 3-dimensional one. The $3$ components of the output of $\Lambda_1$ are the nodes of the layer.
\end{example}

\section{General results}
\label{sec:main}
In order to find out which properties of a neural network can be described by singular Riemannian geometry, it is convenient to analyse the sequence of maps \eqref{def:sequence_of_maps} under the more general \cref{as_1,as_2,Assumption_2}, without requiring the manifolds to be diffeomorphic to $\mathbb{R}^{d_i}$. The notion of pullback introduced in \cref{sec:basic_definitions}, allow us to build sequence of degenerate Riemannian metrics $g^{(i)}$ using the pullback of $g^{(n)}$ through the maps $\Lambda_n,\cdots,\Lambda_n\circ\cdots\Lambda_1$. On any manifold $M_i$, we can extend \cref{def:lenght} and, with abuse of notation, define the pseudolength of a curve $\gamma:(0,1) \rightarrow M_i$ as 
$$
\pl_i(\gamma)~:=~\ds\int_0^1 \sqrt{g^{(i)}_{\gamma(s)}(\dot{\gamma}(s),\dot{\gamma}(s))} ds
$$
and then, considering that $M_i$ is path-connected by hypothesis, we can define the pseudodistance function $\pd_i:M_{i} \times M_{i} \rightarrow \mathbb{R}_0^+$ as:
\begin{equation*}\label{def_pseudometric}
\pd_i(x,y) = \inf \{ \pl_i(\gamma) \ | \ \gamma:[0,1]\rightarrow M_i , \gamma\in\pc \mbox{ s.t. } \gamma(0)=x \mbox{ and } \gamma(1)=y\}   
\end{equation*}
for every $x,y \in M_i$. 
The pair $(M_i,\pd_i)$ is a pseudometric space, which we can be turned into a full-fledged metric space $M_i / \sim_i$ by means of the metric identification $x \sim_i y \Leftrightarrow \delta_i(x,y)=0$. We shall call $\pi_i: M_i \rightarrow M_i / \sim_i$ the map applying the equivalence relation. 
We begin to analyse the sequence of maps \eqref{def:sequence_of_maps} studying the relations between the pseudolenght of a curve in $M_i$ and that of its image through the map $\Lambda_{k} \circ \cdots \circ \Lambda_i$.
\begin{proposition}\label{prop:leghts_relation_every_manifold}
Let $\gamma:[0,1] \rightarrow M_i$ be a piecewise $\mathcal{C}^1$ curve. Let $k \in \{i,i+1,\cdots, n \}$ and consider the curve $\gamma_k = \Lambda_{k} \circ \cdots \circ \Lambda_i \circ \gamma$ on $M_k$. Then $\ell_i(\gamma)=\ell_k (\gamma_k)$
\end{proposition}
\begin{proof}
It is enough to notice that $\gamma_k : (0,1) \rightarrow M_k$ is still a piecewise $\mathcal{C}^1$ curve and that
\begin{equation*}
\begin{split}
\ell_k(\gamma_k) &= \int_0^1\sqrt{g^{(k)}_{\gamma_k(s)}(\dot{\gamma}_k(s),\dot{\gamma}_k(s))} ds \\
&= \int_0^1 \sqrt{((\Lambda_{k} \circ \cdots \circ \Lambda_i)^*g^{(k)})_{\gamma(s)}(\dot{\gamma}(s),\dot{\gamma}(s))} ds\\ &=\ell_i(\gamma)
\end{split}
\end{equation*}
where $(\Lambda_{k} \circ \cdots \circ \Lambda_i)^*g^{(k)}$ is the pullback of $g^{(k)}$ via $\Lambda_{k} \circ \cdots \circ \Lambda_i$.
\end{proof}
Setting $k=n$ in the \cref{prop:leghts_relation_every_manifold} we are considering the map $\mathcal{N}$ and the following result follows immediately.
\begin{corollary}\label{prop:leghts_relation}
Let $\gamma:[0,1] \rightarrow M_i$ be a piecewise $\mathcal{C}^1$ curve. Consider the curve $\Gamma = \mathcal{N}_i \circ \gamma$ on $\mathcal{N}(M_0) \subseteq M_n$. Then $\ell_i(\gamma)=L_n (\Gamma)$, with $L_n$ the length of a curve defined using the Riemannian metric $g^{(n)}$.
\end{corollary}
These two results shall come in handy layer, to study the quotient spaces $M_i / \sim_i$. Note that if $\Lambda_i$ is a map as in the second case of \Cref{as_2}, then $\Lambda_i$ is a diffeomorphism between $M_i$ and its image $\Lambda_i(M_i)$. Being $\Lambda_i$ a diffeomorphism onto its image, the pullback of $g^{(i+1)}$ with respect to $\Lambda_i$ is a (singular) metric of constant rank equal to that of $g^{(i+1)}$. In particular, given any point $p \in M_i$ and a vector $v \in T_p M_i$, we have that $g^{(i)}_p(v,v) = g^{(i+1)}_{\Lambda_i(p)}(d\Lambda_i v,d\Lambda_i v)$. This observation entails that given two points $x,y \in M_i$, they are equivalent if and only if $\Lambda_i(x) = \Lambda_i(y)$. In other words, the map $\Lambda_i$ is merely transporting the equivalence relation on $M_{i+1}$ over $M_i$ without introducing a new quotient of its own. 
For this reason, in this section we focus mainly on maps satisfying the first condition in \Cref{as_2}, corresponding to layers mapping a higher-dimensional manifold to a lower-dimensional one. In addition to transport the equivalence relation already present on $M_{i+1}$, these kind of maps introduce a new quotient of their own. All the results of this section should be read modulo diffeomorphisms with their images given by maps satisfying the second condition of \Cref{as_2}.

General properties of pseudometric spaces allow us to to characterize the spaces $M_i / \sim_i = \pi_i(M_i)$ as follows.
\begin{proposition}
$M_i / \sim_i$ is an open, path-connected, $T_2$, second-countable set.
\end{proposition}
\begin{proof}
An elementary property of quotient maps yields that $M_i / \sim_i$ is still a path-connected space and by \cite[Corollary 3.17]{Pir19}  we also know that $\pi_i$ is an open map, therefore the quotient set $M_i / \sim_i$ is open. Since pseudometric spaces are completely regular \cite[Section 7]{Pir19}, we conclude that $M_i / \sim_i$ is Tychonoff and therefore it is in particular $T_2$. At last we note that, since $\pi_i$ is an open quotient, $M_i / \sim_i$ is also second-countable.
\end{proof}
\begin{proposition}\label{prop_neural_map_equivalence_1}
If two points $p,q \in M_i$ are in the same class of equivalence, then $\mathcal{N}_i(p)=\mathcal{N}_i(q)$.
\end{proposition}
\begin{proof}
Let $p,q \in M_i$ two points in the same class of equivalence $[p]$. Then, since $M_i$ is path connected by hypothesis, there is a piecewise $\mathcal{C}^1$ null curve $\gamma:[0,1] \rightarrow M_0$ connecting $q$ and $p$, with $\ell_i(\gamma)=0$. Consider now the curve $\Gamma = \mathcal{N}_i \circ \gamma$ on $M_n$. By \cref{prop:leghts_relation} we conclude that also $L_n(\Gamma)=0$ and being $g^{(n)}$ a Riemannian metric we have that $\mathcal{N}_i(p)=\mathcal{N}_i(q)$. 
\end{proof}
To prove that the sets $M_i / \sim_i$ are actually smooth manifolds, it is convenient to characterize them in a different way. Let $x,y \in M_i$ and consider the equivalence relation $\sim_{\mathcal{N}_i}$ on $M_i$ defined as follows:
\begin{equation*}
\begin{split}
x\sim_{\mathcal{N}_i} y \textit{ if and only if there is a piecewise } \mathcal{C}^1 \textit{ curve } \gamma:[0,1]\rightarrow M_i \textit{ such that }\\
\gamma(0)=x, \gamma(1)=y \textit{ and } \mathcal{N}_i \circ \gamma(s) = \mathcal{N}_i(x) \ \forall s \in [0,1].
\end{split}
\end{equation*}
The composition of two submersions is still a submersion, therefore \cref{as_2} yields that $\mathcal{N}_i=\Lambda_{n}\circ\cdots\circ\Lambda_i$ is a smooth submersion from $M_i$ to $M_n$. Since $\mathcal{N}_i:M_i \rightarrow M_n$ is a submersion, then $d\mathcal{N}_i:T_p M_i \rightarrow T_{\mathcal{N}_i(p)}M_n$ is surjective for every $p \in M_i$ and any point $a \in M_n$ is a regular value for $\mathcal{N}_i$. By the regular level set theorem \cite[Theorem 9.9]{Tu11}, $\mathcal{N}_i^{-1}(a)$ is a regular submanifold of $M_i$ of dimension equal to $\dim(M_i) - \dim(M_n)$. Notice that given two points $x,y \in M_i$ such that $x \sim_{\mathcal{N}_i}y$, then they lie on a connected component of $\mathcal{N}_i^{-1}(x)$. The reason why we introduced the equivalence relation  $\sim_{\mathcal{N}_i}$ is the following proposition:
\begin{proposition}\label{prop:same_quotient}
Let $x,y \in M_i$, then $x \sim_i y$ if and only if $x \sim_{\mathcal{N}_i} y$.
\end{proposition}
\begin{proof}
If $x \sim_i y$, then there is a piecewise $\mathcal{C}^1$ null curve $\gamma$ with $\gamma(0)=x$  and $\gamma(1)=y$ and we have that $\ell_i(\gamma) = l (\mathcal{N}_i \circ \gamma) = 0$. Since $g^{(n)}$ is a non-degenerate Riemannian metric, $l (\mathcal{N}_i \circ \gamma) = 0$ entails that the tangent vector to $\mathcal{N}_i\circ \gamma (s)$ is the zero vector for every $s \in (0,1)$ and therefore $\mathcal{N}_i\circ \gamma$ is the constant curve $\mathcal{N}_i\circ \gamma(s) = \mathcal{N}_i (x)$. This proves $x \sim_i y \Rightarrow x \sim_{\mathcal{N}_i} y$.
Let us now assume that $x \sim_{\mathcal{N}_i} y$. By definition we know that there is a piecewise $\mathcal{C}^1$ curve $\gamma:[0,1]\rightarrow M_i$ such that $\gamma(0)=x, \gamma(1)=y $ and $\mathcal{N}_i \circ \gamma(s) = \mathcal{N}_i(x) \ \forall s \in [0,1]$. It remains to prove that $\gamma$ is a null curve. This follows from the fact that, being $\mathcal{N}_i \circ \gamma$ a constant curve, then $PL_i(\gamma) = l(\mathcal{N}\circ \gamma)=0$.
\end{proof}
As direct consequences of the above result, we have the following two corollaries. The first is telling us that quotienting $M_i$ by $\sim_i$ or by $\sim_{\mathcal{N}_{i+1}}$ yields the same space, the latter being the converse of \cref{prop_neural_map_equivalence_1}.
\begin{corollary}
$\dfrac{M_{i}}{\sim_{i}} = \dfrac{M_{i}}{\sim_{\mathcal{N}_{i+1}}}$.
\end{corollary}
\begin{corollary}\label{prop_neural_map_equivalence_2}
If two points $p,q \in M_i$ are connected by a $\mathcal{C}^1$ curve $\gamma:[0,1]\rightarrow M_i$ satisfying $\mathcal{N}_i(p)=\mathcal{N}_i \circ \gamma(s)$ for every $s \in [0,1]$, then they lie in the same class of equivalence.
\end{corollary}

The main reason behind the requirement that the maps $\Lambda_i$ are submersions in \cref{as_2} is that we can apply the following proposition \cite{Bou67} (see also \cite{Ser64,Go17,Fer20} for different formulations) to prove that $M_i \sim_i$ is a smooth manifold.
\begin{proposition}[Godement's criterion]\label{prop:godement}
Let $X$ be a smooth manifold and $R \subset X \times X$ be an equivalence relation. The quotient $X/R$ is a smooth manifold if and only if
\begin{itemize}
\item[1)] $R$ is a submanifold of $X \times X$
\item[2)] The projection map on the second component $pr_2:R \subset X \times X \rightarrow X$ is a submersion. 
\end{itemize}
\end{proposition}
Now we can prove that $M_i \sim_i$ is a smooth manifold.
\begin{proposition}\label{prop:characterization_of_quotient}
$\dfrac{M_{i}}{\sim_{i}}$ is a smooth manifold of dimension $dim(\mathcal{N}(M_0))$.
\end{proposition}
\begin{proof}
We prove that the quotient $M_i / \sim_i$ is a smooth manifold using Godement's criterion (\cref{prop:godement}). The graph $\mathcal{G}_{i+1}$ of $\sim_{\mathcal{N}_{i}}$ is the union of $C_p \times C_p$, with $C_p$ a connected component of $\mathcal{N}_i^{-1}(p)$, with $p \in \mathcal{N}_i(M_{i-1}) \subseteq M_n$ and therefore $\mathcal{G}_{i+1}$ is a submanifold of $M_i \times M_i$. Furthermore, the restriction of the projection $pr_2$ to $R$ is the restriction of the identity map to $C_p$ for some $p \in M_i$, which is a diffeomorphism with its image and therefore a submersion. The statement about the dimension follows from the proof of $2)\Rightarrow 1)$ of \cref{prop:godement}, see \cite[Lemma 9.4]{Fer20}, taking in account that $T_p N_i = \dim(\Ker(g^{(i)}))$ is constant.
\end{proof}
This proposition, along with \cite[Lemma 9.4 and Lemma 9.9]{Fer20}, yields that the classes of equivalence $[p]$ are the leaves of a simple foliation of $M_i$ and that $\pi_i$ is a smooth submersion. Gathering all the results obtained above and considering that the tangent vector $\dot{\gamma}(s)$ of a null curve $\gamma:[0,1]\rightarrow M_i$ is in $T_{\gamma(s)}N_i$ for every $s \in [0,1]$, we also have the following lemma. We refer to \ref{appendix} for the notions of integrable distribution and vertical bundle.
\begin{lemma}\label{prop:vertical_integrable}
$\pi_i : M_i \rightarrow M_i/\sim_i$ is a smooth fiber bundle, with $\Ker(d \pi_i) = \mathcal{V}M_i$, which is therefore an integrable distribution. 
\end{lemma}
Using the definitions of equivalence class and vertical bundle the following proposition follows slavishly from the previous lemma.
\begin{proposition}\label{prop:equivalence_characterization}
Every class of equivalence $[p]$ is a path-connected submanifold of $M_i$ and coincide with the fiber of the bundle over $p$.
\end{proposition}

\begin{proposition}\label{prop:N_tilde_smooth_bijective}
There is a surjective smooth map $\widetilde{\mathcal{N}}_i:M_i/\sim_i \rightarrow \mathcal{N}(M_0) \subseteq M_n$ such that the following diagram commutes:
\begin{equation}\label{eq:diagram}
\begin{tikzcd}
M_i \arrow[r, "\mathcal{N}_i"] \arrow[d, "\pi_i"']  & \mathcal{N}(M_0) \\
M_i/\sim_i \arrow[ru, "\widetilde{\mathcal{N}}_i"'] &    
\end{tikzcd}
\end{equation}
\end{proposition}
\begin{proof}
The universal property of quotients yields that there is a unique continuous map $\widetilde{\mathcal{N}}_i:M_i/\sim_i \rightarrow M_n$ such that $ \widetilde{N}_i \circ \pi_i = \mathcal{N}_i$; Being  $\mathcal{N}_i$ and $\pi_i$ smooth maps, also $\widetilde{\mathcal{N}}_i$ is smooth by \cite[Theorem 4.29]{Lee12}. At last, we note that the map $\mathcal{N}$ is surjective by construction.
\end{proof}

In the particular case in which the counter--image of every $a \in M_n$ through $\mathcal{N}_i$ is connected, we have this further result.
\begin{proposition}\label{prop:homeo}
Suppose that $\mathcal{N}_i^{-1}(a)$ is a connected manifold for every $a \in M_n$. Then the map $\widetilde{\mathcal{N}}_i:M_i/\sim_i \rightarrow M_n$ given in \eqref{eq:diagram} is a homeomorphism.
\end{proposition}
\begin{proof}
We already know that $\widetilde{\mathcal{N}}_i$ is continuous and surjective. By definition of $\sim_i$, if $\mathcal{N}_i^{-1}(a)$ is connected for every $a \in M_n$, $\widetilde{\mathcal{N}}_i$  is also injective. Considering that the diagram in \eqref{eq:diagram} commutes, the fact that $\mathcal{N}_i$ is continuous and being $\pi_i$ an open map yield that also $\widetilde{\mathcal{N}}^{-1}_i$ is continuous.
\end{proof}

\section{Applications to Neural Networks}\label{sec:Application to Neural Networks}

We present some applications of the theory previously developed. This section introduces some applications of the theoretical framework; a first in--depth analysis of such applications is carried in \cite{BeMa21}.

\subsection{Equivalence classes of a neural network}\label{subsec:equivalence}
One of the many applications of neural networks is to  gain some information on data in $M_0$, assigning to every point a value establishing how much it satisfies a certain property. A notable example is a classification task: Given a fixed number of classes, one would like to know the probability that a point belongs to a certain class. The notion of equivalence class introduced in \Cref{sec:main} allow us to gain some information on the way the network sees the input manifold $M_0$. In particular, the results of \cref{sec:main} entails the following facts:
\begin{enumerate}[label=\arabic*)]
\item If two points $x,y$ in the input manifold $M_0$ are in the same class of equivalence with respect to $\sim_0$, then $\mathcal{N}(x)=\mathcal{N}(y)$.
\item Given a point $p$ in the output space $M_n$, the counterimage $\mathcal{N}^{-1}(p)$ is a smooth manifold, whose connected components are classes of equivalences in $M_0$ with respect to $\sim_0$. This means that a necessary, but not sufficient, condition for two points $x,y \in M_0$ to be in the same class of equivalence is that $\mathcal{N}(x)=\mathcal{N}(y)$.
\item Any class of equivalence $[x]$, $x \in M_0$, is a maximal integral submanifold of $\mathcal{V}M_0$.
\end{enumerate}
Together, these three facts allow us to build the equivalence class of a point $x \in M_0$. Since $\mathcal{V}M_0$ is an integrable distribution by \Cref{prop:vertical_integrable} -- with $\mathcal{V}M_0 = \Ker(g^{(0)})$ -- then we can find $\dim(\Ker(g^{(0)}))$ vector fields which are a basis of $T_x M_0$ for every point $x \in M_0$ or, in other words, they span $\Ker(g^{(0)}_x)$ over each point $x \in M_0$. In particular this means that we can find $\dim(\Ker(g^{(0)}_x))$ linearly independent eigenvectors of $g^{(0)}_x$ associated with the null eigenvalue depending smoothly on the point $x$, a result which is in general non-trivial if a matrix depends on several parameters \cite{Rell69,Kato66}. Now we can build all the null curves, since we know the null vectors over each point. By definition, a null curve passing through a point $p \in M_0$ is a solution of the following Cauchy problem:
\begin{equation}\label{eq:cauchy_null}
\begin{cases}
\dot{\gamma} = v\\
\gamma(0) = p
\end{cases}
\end{equation}
with $v$ a smooth vector field such that $v_x \in \Ker(g^{(0)}_x)$ for every $x \in M_0$. An equivalence class $p$ belongs to is the set given by the union of all the null curves passing through $p$. From a theoretical point of view, we solved the problem of building the classes of equivalence of $M_0$.
The observation about the linearly independent eigenvectors of $g^{(0)}_x$ associated with the null eigenvalue depending smoothly on the point $x$ has an important practical consequence in the case $M_0 = \mathbb{R}^{d_0}$, as allows one to build an approximation of an equivalence class by triangulations or, in other words, a way to numerically build  an approximation of the connected components of $\mathcal{N}^{-1}(p)$ for some point $p \in M_n$. We discuss the SiMEC algorithm -- at least in the one-dimensional case -- in \cite{BeMa21}.

\begin{remark}
The same reasoning can be applied to the other manifolds $M_i$, in order to construct the equivalence classes of each manifold of the sequence realizing the neural network.
\end{remark}

\subsection{Level curves}\label{subsec:level}

Suppose now to consider a non-linear regression task in which we trained a neural network to approximate a non-linear function $f: M_0\rightarrow M_n$. By definition a level set of $f$ is the set of the points $x \in M_0$ such that $f(x)=k$ for some $k$ in the range of $f$. Comparing this definition with that of equivalence class of $x \in M_0$, we note that, whenever a level curve is a connected set, then it coincides with the set of the points which are in the same class of equivalence of $x$. In general, if we consider non-connected level sets, then each connected component of a level set is a class of equivalence. For example, let us consider the Ackley function
\begin{equation}
a(x,y)=-20 \exp \left[-0.2{\sqrt {0.5\left(x^{2}+y^{2}\right)}}\right]-\exp \left[0.5\left(\cos 2\pi x+\cos 2\pi y\right)\right]+e+20
\end{equation}
Every connected component of the level set of $f$ is an equivalence class.
\begin{figure}[h!]
\centering
\begin{subfigure}[t]{0.45\textwidth}
\includegraphics[width=\textwidth]{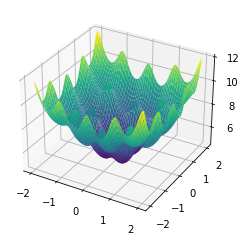} 
\caption{Plot of the Ackley function.}
\end{subfigure}
\begin{subfigure}[t]{0.45\textwidth}
\includegraphics[width=\textwidth]{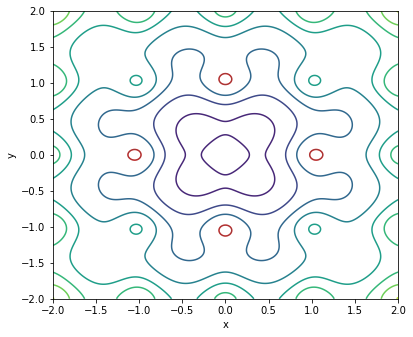} 
\caption{Contour plot of the Ackley function. The four curves in red belong to the same level set, but they form different equivalence classes.}
\end{subfigure}
\label{fig:ackley_function}
\end{figure}

To clarify the connection between equivalence classes and level curves, we present the following toy example. Consider the function $f(x_0,x_1) = \ln\left(1+(1+e^{2x_0+x_1})^4 \right)$. In general, by the universal approximation theorem, neural networks are only able to approximate functions, under suitable hypothesis on the structure of a neural network and on the function we want to learn \cite{Cyb89,HSW89,Han19,Kid20,Kr20,Kr21}. We made this particular choice because a neural network can learn this function exactly. Indeed, if we consider a shallow network of the form
\begin{equation}
\label{eq:example}
\begin{tikzcd}
M_{0} \arrow[r, "\Lambda_1"] & M_{1} \arrow[r, "\Lambda_2"] & M_{2}
\end{tikzcd}
\end{equation}
with $M_0 = \mathbb{R}^2$, $M_1 = \mathbb{R}$ and $M_2=\mathbb{R}$, such that $\Lambda_1$ and $\Lambda_2$ are fully connected softplus layers whose matrices of weights are given by 
\begin{equation}\label{eq:Exweigths}
A^{(1)} = 
\begin{pmatrix}
2 & 1 \\ 
1 & 0
\end{pmatrix}
\quad 
A^{(2)} = \begin{pmatrix}
4 & 0
\end{pmatrix} 
\end{equation}
and such that all the biases are null, then the neural network map $\mathcal{N}=\Lambda_2 \circ \Lambda_1$ in Cartesian coordinates is the function $f$. See \Cref{fig:example} for a classical visual inspection of the network employed in this example.
\begin{figure}[hbtp]
	\begin{center}
	\includegraphics[width=0.8\textwidth]{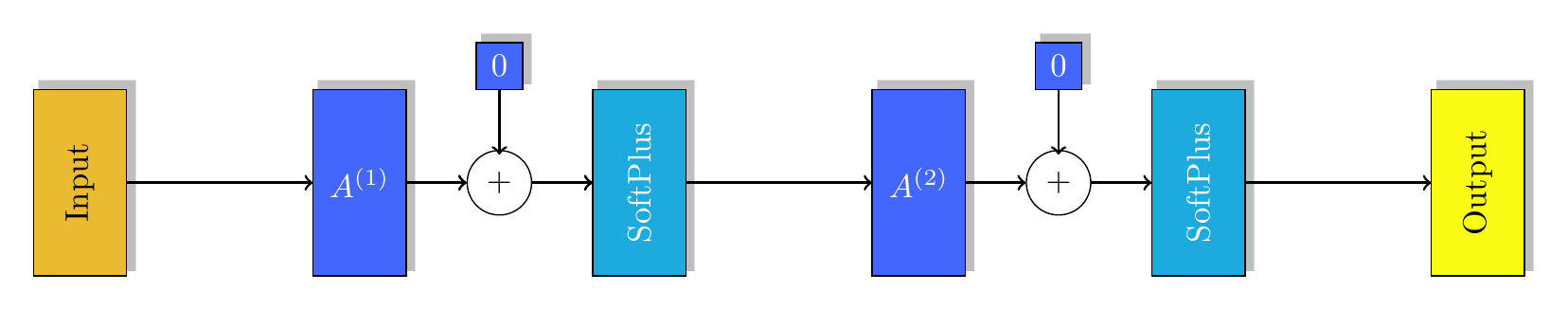}
	\end{center}
	\caption{Structure of the neural network employed in \Cref{eq:example} with weights in \Cref{eq:Exweigths}. The biases are set to 0.}
	\label{fig:example}
\end{figure}
By \Cref{prop_neural_map_equivalence_2}, an equivalence class of $\mathcal{N}$ contains a level curve of $f$, which in this simple case we can explicitly compute: The result is that the level curves abide to the equation $2x_0+x_1 = k$, $k \in \mathbb{R}$. We want to show that the observation above allow us to recover these level curves building the curves of the equivalence classes, whose tangent vector is by definition in $\Ker(g^{(0)})$, $g^{(0)}$ being the pullback a metric $g^{(2)}$ on the output manifold $M_2$ to the input manifold.

Suppose to endow $M_2$ with the Riemannian metric $g^{(2)} = (1)$ -- i.e. we consider $\mathbb{R}$ with its standard scalar product, whose associated matrix is the $1 \times 1$ matrix whose only entry is $1$, \emph{i.e.} the identity operator. After a short computation, we find that the pullback metric $g^{(0)} = \mathcal{N}^* g^{(2)}$  is given by:
\begin{equation}
g^{(0)} = \begin{pmatrix}
\xi \psi^2 & \chi \psi \omega\\ 
\xi \psi \omega & \xi \omega^2
\end{pmatrix}  
\end{equation}
where
\begin{equation}
\xi = \dfrac{16 e^{8y}}{(1+e^{4y})^2}, \quad \psi = \dfrac{2e^{2x_0+x_1}}{1+e^{2x_0+x_1}}, \quad \chi = \dfrac{e^{2x_0+x_1}}{1+e^{2x_0+x_1}}
\end{equation}

In this case, $\Ker(g^{(0)})$ is one-dimensional, since $g^{(0)}$ has rank $1$. A straightforward computation yields that
\begin{equation}
\Ker(g^{(0)}) = \Span \left\{ \begin{pmatrix}
1 \\ 
-2
\end{pmatrix}   \right\}
\end{equation}
By definition a null curve $\gamma(t)=(x_0(t),x_1(t))$ in $M_0$ passing through $(p,q) \in M_0$ satisfies the following system:
\begin{equation}
\begin{cases}
\dot{x}_0 = 1\\
\dot{x}_1 = -2\\
x_0(0) = p\\
x_1(0) = q
\end{cases}
\end{equation}
The unique solution of this system of ODEs is given by
\begin{equation}
\begin{cases}
x_0(t) = t+p\\
x_1(t) = q -2t
\end{cases}, \quad t \in \mathbb{R}
\end{equation}
from which we can find the Cartesian equation of $\gamma$, that is
\begin{equation*}
x_1= -2x_0 + q +2p
\end{equation*}
Comparing this result with the general form of a level curve of $f$ we gave before, the curve $\gamma$ is exactly the level curve passing through $(p,q)$.

\subsection{Exploring the space of weights and biases without changing the output}\label{subsec:exploring_weights}

Up to now, we always considered neural networks with a fixed choice of weights and biases. If we allow them to change, as it happens during the training, the layers $\Lambda_i$ are actually functions of the inputs $x^{(i)}$, of the weights through the matrix $A^{(i)}$ and of the vector of the biases $b^{(i)}$. Therefore, we can consider a neural network as a sequence of maps $\Lambda^{(i)} : M_{i-1} \times \Omega_i \rightarrow M_i$, with $\Omega_i$ the space in which the weights and the biases of the $i$--th layer take values. To stress the dependence on $A^{(i)}$ and $b^{(i)}$, we employ the notation $\Lambda^{(i)}\{A^{(i)},b^{(i)}\}$ to denote a layer with fixed biases and weights given by the vector $b^{(i)}$ and by the matrix $A^{(i)}$ respectively. Now we discuss how to apply the ideas of \Cref{subsec:equivalence,subsec:level} to the space $\Omega_1$ of the weights and biases of the first layer $\Lambda_1$. For simplicity, we limit ourselves to discuss how to explore $\Omega_1$. Suppose to fix the weights and the biases of all the layers but the first one, in which we allow them to vary in $\Omega_1 = \mathbb{R}^k$ for a suitable $k \in \mathbb{N}$ depending on the structure of the layer. Suppose to start with a matrix of weights $A_1$ and a vector of biases $b_1$. Consider a point $p \in M_1$. Since $p$ is a fixed point, while $A_1$ and $b_1$ can vary, we can see the first layer as a map from $\Omega_1$ to $M_1$, namely we consider a neural network as a sequence of maps between manifolds of the form:
\begin{equation}
\begin{tikzcd}
\Omega_{1} \arrow[r, "\Lambda_1"] & M_{1} \arrow[r, "\Lambda_2"] & M_{2} \arrow[r,"\Lambda_4"]  & \cdots \arrow[r,"\Lambda_{n-1}"] & M_{n-1} \arrow[r, "\Lambda_n"] & M_n
\end{tikzcd}
\end{equation}
This sequence is nothing but \eqref{def:sequence_of_maps_intro}, with the difference that we see $\Lambda_1$ as a function of its (non-fixed) weights and biases instead of a function over $M_0$.
We define the class of equivalence $[A^{(1)},b^{(1)}]_p$ of $(A^{(1)},b^{(1)})$ in the space $\Omega_1$ (with respect to the point $p$) as the subset of all the weights and biases $(\tilde{A}^{(1)},\tilde{b}^{(1)})$ such that $\Lambda^{(1)}\{A^{(1)},b^{(1)}\}(p) = \Lambda^{(1)}\{\tilde{A}^{(1)},\tilde{b}^{(1)}\}(p)$. Since the output of the first layer applied to the point $p$ does not change, neither does the output of the whole neural network. For example we may change weights and biases of a neural network whose first layer had weights $A^{(1)}$ and biases $b^{(1)}$ without changing how a point $p \in M_1$ is classified, as long as we take them in the equivalence class $[A^{(1)},b^{(1)}]_p$.

Following the analysis carried out in \Cref{sec:main} and in \Cref{subsec:equivalence}, this time with the space $\Omega_1$ instead of the manifold $M_1$, we know that to build the class of equivalence $[A^{(1)},b^{(1)}]_p$ of a point $(A^{(1)},b^{(1)}) \in \Omega_1$ we have to look for null curves passing through $(A^{(1)},b^{(1)})$.

We illustrate the procedure with the following example.
Consider a shallow neural network
$$
\begin{tikzcd}
M_{0} \arrow[r, "\Lambda_{1}"] & M_{1} \arrow[r, "\Lambda_{2}"] & M_{2}
\end{tikzcd}
$$
with $M_0 = \mathbb{R}^2$, $M_1 = \mathbb{R}$ and $M_2=\mathbb{R}$, such that $\Lambda^{(1)}$ and $\Lambda^{(2)}$ are fully connected softplus layers whose matrices of weights are given by 
\begin{equation*}
A^{(1)} = 
\begin{pmatrix}
a & b \\ 
\end{pmatrix}
\quad 
A^{(2)} = \begin{pmatrix}
c
\end{pmatrix} 
\end{equation*}
and with biases 
\begin{equation*}
b^{(1)} = 
\begin{pmatrix}
d
\end{pmatrix}
\quad
b^{(2)} = 
\begin{pmatrix}
0
\end{pmatrix}
\end{equation*}
To simplify the computations, we do not allow $b^{(2)}$ to change.
Given a point $(x_0,x_1) \in M_0$, we have:
\begin{equation*}
\Lambda_{1}\{A^{(1)},b^{(1)}\}(x_0,x_1) = \ln(1+\exp(ax_0+bx_1+d))
\end{equation*}
and
\begin{equation*}
\Lambda_2\{A^{(2)},b^{(2)}\}(y_0) = \ln(1+\exp(cy_0)), \quad y_0 = \Lambda_{1}\{A^{(1)},b^{(1)}\}(x_0,x_1)
\end{equation*}
Therefore
\begin{equation*}
\mathcal{N}(x_0,y_0) = \ln \left( 1 + \left(1+\exp(ax_0+bx_1+d) \right)^c \right)
\end{equation*}
We want to find all the matrices $\tilde{A}_1$ and all biases $\tilde{b}_1$ such that $\Lambda_1\{A_1,b_1\}(x_0,x_1)=\Lambda_1\{\tilde{A}_1,\tilde{b}_1\}(x_0,x_1)$. 
Now we proceed as in in the example at \cref{subsec:equivalence}, with the difference that the Jacobian $J_{\Lambda^{0}}$ we use to get the pullback of the final metric $g^{(2)}=(1)$ is computed differentiating with respect to the weights and the biases -- this time the point $(x_0,x_1)$ is fixed, while weights and biases are not; Namely we are seeing the neural network as the following map
$$
\begin{tikzcd}
\Omega_{1} \arrow[r, "\Lambda_1"] & M_{1} \arrow[r, "\Lambda_2"] & M_{2}
\end{tikzcd}
$$
After some computations, we get that $\Ker(g_0)$, if $x_0\neq 0$, is given by the vectors of the form
\begin{equation*}
\begin{pmatrix}
- \alpha \dfrac{x_1}{x_0} - \dfrac{\beta}{x_0}\\
\alpha\\
\beta
\end{pmatrix}
\quad
\alpha,\beta \in \mathbb{R}
\end{equation*}
A null curve in $\gamma(t)=(\tilde{a}(t),\tilde{b}(t),\tilde{d}(t))$ in $\Omega_1$ passing through $(A_1,b_1)=(a,b,d)$ is a solution of the following Cauchy problem:
\begin{equation*}
\begin{cases}
\dot{\tilde{a}} = - \alpha \dfrac{x_1}{x_0} - \dfrac{\beta}{x_0}\\
\dot{\tilde{b}} = \alpha\\
\dot{\tilde{d}} = \beta \\
\tilde{a}(0) = a\\
\tilde{b}(0) = b\\
\tilde{d}(0) = d
\end{cases}
\end{equation*}
whose unique solution is given by 
\begin{equation*}
\begin{cases}
\tilde{a}(t) =  -\dfrac{\alpha x_1}{x_0} t - \dfrac{\beta}{x_0}t + a \\
\tilde{b}(t) = \alpha t +b\\
\tilde{d}(t) = \beta t + d
\end{cases} \quad t \in \mathbb{R}
\end{equation*}
Calling $\tilde{A}^{(1)}(t)=(\tilde{a}(t),\tilde{b}(t))$ and $\tilde{b}^{(1)}(t)=(\tilde{d}(t))$, after a short computation we find that
\begin{equation*}
\Lambda_1\{\tilde{A}^{(1)}(t),\tilde{b}^{(1)}(t)\}(x_0,x_1) = \ln(1+\exp(\tilde{a}(t)x_0+\tilde{b}(t)x_1+\tilde{d}(t))) = \ln(1+\exp(ax_0+bx_1+d))
\end{equation*}
for every $t \in \mathbb{R}$. Therefore, changing weights and bias of the first layer inside the class of equivalece $[A^{(1)},b^{(1)}] \subset \Omega_1$ does not change the output of $\Lambda^{(1)}$ and of the neural network.
This topic is worth of further enquiries, as it may be employed in the training process and we are planning to study it in future investigations.

\section{Conclusions}
In this work, motivated by the geometric definition of neural network proposed in \cite{HaRa17} and \cite{Shen18}, we propose a theoretical framework which can be employed to study neural networks. In particular we generalize the main idea of \cite{HaRa17}, allowing a Riemannian metric to be degenerate. Endowing the output manifold with a Riemannian metric and computing its pullback to the input manifold, we obtain a singular Riemannian metric on the latter, which in turns induces a pseudodistance between input points. An interesting property is that the pseudodistance between two input points which are mapped by the network to the same output is zero. Then we considered the Kolmogorov quotient of the input manifold by of the pseudodistance, finding that -- assuming the full rank hypothesis -- the resulting space has the same dimension of the output manifold and that the pullback of the Riemannian metric on the output manifold is still a Riemannian metric. With respect to this quotient, all the input points mapped to the same output belong to the same equivalence class. Furthermore, the singular Riemannian metric induces a natural decomposition of the tangent bundle of the input manifold as a direct sum of a vertical and a horizontal bundle, the latter being an integrable distribution. This fact allow us to employ the singular Riemannian metric of the input manifold to characterize the equivalence classes of the quotient space. 
An interesting example, in a non-linear regression task, is the study of the level curves of the function learned by the neural network, as all the point on a level curve are in the same equivalence class by construction. The theoretical framework introduced in this paper, in addition to give an abstract characterization of the equivalence classes, also yields a way to build such equivalence classes. We note that the results of \cref{sec:main,sec:Application to Neural Networks} suggest a numerical algorithm allowing one to build the preimages in $M_0$ of a point $p \in M_n$ through the map $\mathcal{N}$. In particular, in addition to study the applications  discussed above, it would be interesting to see if this framework can be employed to build the separating hypersurfaces in a classifier neural network. 

We can extend this work in several directions. We can discuss how to extend these results to other kind of neural networks, like convolutional networks -- at least with average pooling. The increased complexity of these kind on neural networks compared to those analysed in \cite{Shen18} and \cite{HaRa17} may give rise to other geometric structures associated with the manifolds of the sequence \eqref{def:sequence_of_maps}, for example some symmetry groups, which can be exploited to gain other information about the way the input manifold is seen by a neural network.  Another interesting line of research is to extend the presented work and verify whether the result of this paper hold true also for some ReLu layers, using a sequence of functions converging to the ReLu activation function.

At last, it would be interesting to endow the final manifold $M_n$ with a pseudoconnection and to study the properties of its pullback to the other manifolds of the sequence and to the quotient manifolds $M_i/\sim_i$, focusing in particular on the curvature and on the choice of a horizontal bundle associated with the resulting pseudoconnection. The study of the curvature of the manifolds of a neural network seen as a sequence of maps could shed some light on the inner working of a neural network.

\section*{Acknowledgments}

We are grateful to Claudio Dappiaggi for the useful discussions. This work has been partially supported by Project Interreg GE.RI.KO.--MERA. 

\begin{appendices}
\appendix
\section{Integrable distributions and vertical bundles.}\label{appendix}
In this appendix we introduce some advanced notions of differential geometry we need in the proofs of \Cref{sec:main}.
We begin with the notion of distribution.
\begin{definition}
A distribution $\mathcal{D}$ of dimension $k$ over a m-dimensional manifold $M$ is a collection of $k$ smooth vector fields $v_1,\cdots,v_k$ such that $(v_1)_p,\cdots,(v_k)_p$ form a basis of a vector subspace of dimension $k$ in $T_p M$ for every $p \in M$.
\end{definition}
Given a manifold of dimension $m$, we can find always find $m$ smooth vector fields $v_1,\cdots,v_m$ such that $T_p M$ is generated by $(v_1)_p,\cdots,(v_k)_p$. However, the converse is not always true: If we take $m$ vector fields $v_1,\cdots,v_m$, it may happen that there is no manifold $M$ such that $v_1,\cdots,v_m$ are generating $T_pM$ for every $p \in M$. When it happens we have an integrable distribution.
\begin{definition}
A distribution $\mathcal{D}$ of dimension $k$ is an integrable distribution if there exist a manifold $M$ of dimension $m \geq k$ such that the collection of $k$ smooth vector fields $v_1,\cdots,v_k$ are generating a vector space of dimension $k$ over $T_p M$ for all $p \in M$.
\end{definition}
Another useful definition is that of (trivial) fiber bundle.
\begin{definition}
A trivial fiber bundle is a structure $(E,B,\pi,F)$, where $E, B$ and $F$ are topological spaces with $E=B\times F$ and the map $\pi : E \rightarrow B$ is the projection of $B \times F$ on $B$. The space $F$ is called typical fiber. In the case $F$ is a vector space, then $(E,B,\pi,F)$ is called a trivial vector bundle.
\end{definition}
When there is no ambiguity, we simply say that $E$ is a fiber bundle. This happens when there is a natural choice for $B,\pi$ and $F$. For example when we say that $TM$ is the tangent bundle over $M$, we implicitly mean that $(M,M \times \mathbb{R}^n,\pi,\mathbb{R}^n)$ is a vector bundle.
\begin{remark}
The general definition of (non-trivial) fiber bundle is quite technical and it is not needed to understand the proofs of \Cref{sec:main}. The interested reader can find it, at least in the case in which $F$ is a vector space, in \cite{Tu11}.
\end{remark}
At last, we introduce the horizontal and the vertical bundles.
\begin{definition} 
Let $(E,M,\pi,F)$ be a vector bundle over a manifold $M$. Then the vertical space $\mathcal{V}_p E$ at $p \in E$ is the vector space $\mathcal{V}_p E=\Ker(d_p\pi) \subset T_pE$. The horizontal space $H_p E$ is a choice of a subspace of $T_p E$ such that $T_p E = \mathcal{V}_p E \oplus \mathcal{H}_p E$. The spaces $\mathcal{V}E:= \sqcup_{p \in E} \mathcal{V}_p E$ and $\mathcal{H}E:= \sqcup_{p \in E} \mathcal{H}_p E$
are two bundles called vertical and horizontal bundles respectively.
\end{definition}
An important property of the vertical bundle is that it is integrable in the following sense. Suppose that $\mathcal{VM}$ is a k-dimensional space, then chosen a collection of $k$ smooth vector fields valued over $M$ which spans $\mathcal{V}_p M$ at every point $p \in M$, we can find a submanifold $\mathcal{V} \subset M$ such that $T\mathcal{V} = \mathcal{V}M$. Suppose now to consider a smooth $m$-dimensional manifold $M$ equipped with a singular Riemannian metric $g$. If $\dim \Ker(g(x)) = r \in \mathbb{R}$ for every $x \in M$, namely $g$ is of constant rank, then \cite{Bel74} the singular metric $g$ induces a splitting of the tangent bundle $TM$ into a vertical and a horizontal bundles of dimensions $r$ and $m-r$ respectively.

\end{appendices}



\bibliographystyle{elsarticle-num} 
\bibliography{biblio.bib}


\end{document}